\documentclass[twoside]{article}
\usepackage[accepted]{aistats2018}
\usepackage{amsmath,amsthm}
\usepackage{mathtools,bm,dsfont,amsfonts,amssymb}
\usepackage{graphicx}
\usepackage{color}
\usepackage{epstopdf}

\long\def\comment#1{}

\newtheorem{theorem}{Theorem}
\newtheorem{corollary}{Corollary}[theorem]
\newtheorem{lemma}[theorem]{Lemma}

\newcommand{\E}{{\mathbb{E}}}
\newcommand{\expect}[2]{\E_{#1}\left[{#2}\right]}
\newcommand{\expectcond}[3]{\expect{#1}{{#2}\,|\,{#3}}}

\newcommand{\inner}[1]{\langle {#1} \rangle}

\newcommand{\R}{\mathbb{R}}
\newcommand{\1}[1]{\mathds{1}_{\{{#1}\}}}
\renewcommand{\P}{\mathbb{P}}

\newcommand{\probsub}[2]{{\P}_{#2}\left[{#1}\right]}

\newcommand{\probcondsub}[3]{{\probsub{#1\,|\,#2}{#3}}}

\newcommand{\method}[1]{{\textsc{#1}}} 
\newcommand{\secref}[1]{Sec. \ref{#1}}

\newcommand{\ourmthd}{{\method{IntPred}}}

\newcommand{\X}{{\mathcal{X}}}
\newcommand{\Y}{{\mathcal{Y}}}

\newcommand{\yhat}{{\hat{y}}}
\newcommand{\lhat}{{\hat{\ell}}}
\newcommand{\uhat}{{\hat{u}}}
\newcommand{\Deltahat}{{\hat{\Delta}}}

\renewcommand{\L}{\mathcal{L}}

\newcommand{\F}{\mathcal{F}}
\newcommand{\N}{\mathcal{N}}
\newcommand{\B}{\mathcal{B}}

\renewcommand{\H}{\mathcal{H}}
\renewcommand{\l}{{\ell}}
\renewcommand{\u}{{u}}

\newcommand{\intrvltmplt}[2]{{[ {#1} , {#2} ]}}
\newcommand{\intrvl}{\intrvltmplt{\l}{\u}}
\newcommand{\intrvlhat}{\intrvltmplt{\lhat}{\uhat}}

\renewcommand{\sup}[2]{{{#1}_{#2}}}
\newcommand{\xsup}[1]{\sup{x}{#1}}
\newcommand{\ysup}[1]{\sup{y}{#1}}
\newcommand{\xsupi}{{\xsup{i}}}
\newcommand{\ysupi}{{\ysup{i}}}
\newcommand{\lhati}{\sup{\lhat}{i}}
\newcommand{\uhati}{\sup{\uhat}{i}}
\newcommand{\Loss}[1]{{\L \left( {#1} \right)}}

\newcommand{\losssub}[2]{{L_{#1} \left( {#2} \right)}}

\newcommand{\lprxy}{{\tilde{L}}}
\newcommand{\lossprxy}[1]{{\lprxy \left( {#1} \right)}}
\newcommand{\loss}[1]{{\losssub{}{#1}}}
\newcommand{\fprxy}{{\tilde{f}}}
\newcommand{\budget}{{B}}
\newcommand{\conf}{{\alpha}}
\newcommand{\vcdim}{{k}}
\newcommand{\Base}{{\B}}
\newcommand{\base}{{b}}
\newcommand{\Thresh}{{\H}}
\newcommand{\thresh}{{h}}
\newcommand{\maxfs}{{\textsc{Max-FS}}}

\begin{document}

%

%

\twocolumn[

\aistatstitle{Discriminative Learning of Prediction Intervals}

\aistatsauthor{Nir Rosenfeld  \And  Yishay Mansour  \And  Elad Yom-Tov}

\aistatsaddress{Harvard University \\ and Microsoft Research
    \And  Tel Aviv University  \And  Microsoft Research} ]

\begin{abstract}
In this work we consider the task of constructing prediction intervals
in an inductive batch setting.
We present a discriminative learning framework 
which optimizes the \emph{expected} error rate under
a budget constraint on the interval sizes.
Most current methods for constructing prediction intervals
offer guarantees for a single new test point.
Applying these methods to multiple test points can result in 
a high computational overhead and degraded statistical guarantees.
By focusing on expected errors, our method allows for variability
in the per-example conditional error rates.
As we demonstrate both analytically and empirically,
this flexibility can increase the overall accuracy,
or alternatively, reduce the average interval size.

While the problem we consider is of a regressive flavor,
the loss we use is combinatorial.
This allows us to provide PAC-style, finite-sample guarantees.
Computationally, we show that our original objective is NP-hard,
and suggest a tractable convex surrogate.
We conclude with a series of experimental evaluations. 
\end{abstract}


\section{Introduction} \label{sec:intro}

Constructing an interval which contains some point
of interest with high probability is a fundamental task in statistics.
In contrast to point predictions, intervals provide some measure
of confidence, and can be used to reason about the reliability of a prediction.
In this paper we focus on \emph{prediction intervals} (PIs).
For some predetermined \emph{significance level} $\conf$,
the classic task of PI estimation is to construct an interval 
$\intrvl$ which will contain
a point $y \in \R$ sampled from some distribution $D$
with probability of at least $1-\conf$.
PIs are hence similar in spirit to \emph{confidence intervals},
but are designed to contain future sampled points rather than 
population parameters such as the mean or variance.

In this paper we consider non-parametric PI estimation
in a regressional setting.
Let $D=D_{XY}$ be an unknown joint distribution over
examples $x \in \X$ and labels $y \in \Y = \R$,
where $X,Y$ denote random variables and $x,y$ their instantiations.
In the classic PI task, we are given a training set
$S=\{ (\xsupi, \ysupi) \}_{i=1}^m$ of $m$
pairs sampled i.i.d. from $D_{XY}$,
and an additional test example $\xsup{m+1}$ sampled from the marginal $D_X$. 
Then, for a given confidence level $\conf$,
the goal is to construct an interval
$\intrvl \in \R^2$ 
which will contain the true label $\ysup{m+1} \sim D_{Y|X=x}$ 
with probability of at least $1-\conf$.
The classic approach to this task is to first estimate
the conditional $D_{Y|X=\xsup{m+1}}$,
and then infer the smallest interval which covers $1-\conf$ of the probability mass.
In the case of a Gaussian conditional distribution, this
leads to simple closed form solution for computing PIs,
and provides asymptotic guarantees when the data is indeed Gaussian
(see also Sec. \ref{sec:related}).
There are many variations on this parametric two-step approach
\cite{koenker1978regression,koenker2001quantile},
as well as other direct and non-parametric methods
\cite{stine1985bootstrap,koenker2005quantile,steinberger2016leave}.
%

Methods such as the above are designed for
the classic setting which concerns only a single new test example.
In many realistic cases, however, the task
is to predict intervals for a \emph{set} of future test points,
given only the training data.
This setting is known as \emph{inductive batch learning}
and has been extensively studied in machine learning.
As we later discuss in detail, applying single test-point methods
to the batch setting is rarely straightforward,
and often comes at the cost of a significant computational overhead 
and/or the loss of statistical guarantees
\cite{lei2013distribution,lei2014distribution,chen2016trimmed}.
This is also true for the recently popular methods based on conformal prediction
designed in the online learning setting \cite{burnaev2014efficiency,lei2015conformal}.

Our goal in this paper is to provide a general framework
for efficient PI estimation in the batch setting.
The approach we present formulates PI estimation
as a discriminative learning task,
where the goal is to learn an interval predictor with high expected accuracy.
To this end, we design a learning objective which takes into 
account the natural tradeoff between accuracy and interval size.
Our analysis provides computational guarantees on training such predictors,
as well as statistical guarantees on their generalization.
These ensure that
an interval predictor with an average error of $\conf$ on the training set
is guaranteed to have an expected error of roughly $\conf$ on the entire population.


The method we present differs from the standard single-point approach in three important ways.
First,
the learning objective we consider relaxes the classic $\conf$-confidence requirement
to hold \emph{in expectation}.
Focusing on expected errors allows our method to
predict intervals with different error rates for different examples.
Second,
a model-free discriminative framework separates
the task (predicting accurate intervals) from the model (Gaussian, linear, etc.).
To this end, we first identify an appropriate combinatorial loss function,
and then suggest a tractable convex surrogate.
This allows for trading off guarantees on optimization (in the linear case)
with potentially higher accuracy (via more expressive predictor classes). 
Third,
our approach allows for reversing the classic roles of accuracy and interval size:
instead of fixing an error rate of $\conf$ and predicting tight intervals,
we can minimize the error under a fixed \emph{budget} constraint
on the mean interval size.
Intuitively,
allowing for variability in individual interval sizes
enables the overall accuracy to be boosted by ``sacrificing'' some points
for the sake of others.
Overall, the formulation we present is geared towards maximizing expected accuracy,
rather than providing (asymptotic) worst-case guarantees.


For the analysis,
we partition the significance level $\conf$
into a confidence term $\delta$ (over the train set $S \sim D_{XY}^m$)
and an accuracy term $\epsilon$ (over new points $y \sim D_{Y|X}$),
thus promoting PAC-style results. 
Fixing a base class $\Base$ of real functions,
we show how the VC dimension of interval predictors based on $\Base$
can be expressed in terms of the VC dimension of thresholds over $\Base$.
This covers several well-established classes of binary classifiers.
Our results show that the VC of intervals is $O(d)$, where
$d$ is the VC of the corresponding threshold class.
Thus, our framework provides finite-sample bounds that depend only on the
function class.
This is in contrast 
to classic methods which typically offer asymptotic guarantees that
depend on the algorithm.


The paper is structured as follows.
In \secref{sec:related} we review related work on prediction intervals.
We then compare the single-point and batch settings in \secref{sec:sp_vs_batch},
and present our discriminative method in \secref{sec:method}.
\secref{sec:theory} contains a theoretical analysis of both
the statistical complexity (\secref{subsec:vcdim})
and computational complexity (\secref{subsec:nphard}) of our approach.
Finally, \secref{sec:experiments} contains a detailed experimental
evaluation of our method on a collection of benchmark datasets.
We conclude with a discussion of our results in \secref{sec:discussion}.

%
%
%
%
%


\section{Related Work} \label{sec:related}
The simplest approach for constructing PIs includes two steps.
In the first step, as in standard regression,
a point-predictor $f:\X \rightarrow \R$ is trained.
Then, intervals $\intrvl$ are constructed around the predictions $\yhat=f(x)$.
In the parametric setting,
the conditional distribution $D_{\Y|\X}$ is assumed to have
some parametric form
(for a comprehensive review, see \cite{geisser1993predictive}).
Given a new example $x$,
the interval boundaries $\l$ and $u$ are then set
to contain $1-\conf$ of the probability mass.
Many parametric forms of $D_{\Y|\X}$ lead to closed form solutions
for $\intrvl$.
For example, assuming $y \sim \N(\mu(x),1)$ gives:
\begin{equation}
\intrvl = \hat{\mu}(x) \pm \Phi^{-1}( \alpha/2)
\sqrt{\frac{m+1}{m}}
\label{eq:gaussian}
\end{equation}
where $\Phi^{-1}(\cdot)$ is the inverse normal cdf,
and the conditional mean estimates $\yhat=\hat{\mu}(x)$
can be learned via simple regression \cite{seber2012linear}.

In Eq. \eqref{eq:gaussian}, the interval boundaries are
in effect set to be the $\conf/2$ and $1-\conf/2$
quantiles of the (assumed) conditional distribution.
This suggests that the above two-stage process
can be replaced with a direct estimation of the conditional quantiles $q_\tau(x)$.
Quantile regression 
can then be used to predict PIs \cite{koenker2005quantile},
circumventing the need to explicitly state a distributional assumption.
This can either be done by minimizing a tilted version of the absolute loss
over a parametric hypothesis class \cite{koenker1978regression,koenker2001quantile},
or by using non-parametric empirical quantile estimates using
bootstrapping \cite{stine1985bootstrap} or
leave-one-out sampling \cite{steinberger2016leave}.

Methods such as the above are often justified by asymptotic guarantees.
These, however, do not always lead to good performance in practice.
Some methods offer corrections for finite-sample biases
\cite{schmoyer1992asymptotically,olive2007prediction},
but are often based on further assumptions,
and general sample complexity results are typically hard to obtain.

In a set of comprehensive papers on \emph{conformal prediction},
the authors expand the classic single test-point setting to an online setting
where a set of new examples are given in sequence \cite{gammerman2007hedging,shafer2008tutorial}.
The goal is then to minimize regret, namely predict intervals
which will be good in hindsight compared to the best fixed alternative.
Based on this,  recent work \cite{lei2017distribution} introduces
flexible extensions to the basic conformal method,
as well as additional finite-sample and distribution-free asymptotic guarantees.
While several works on conformal prediction discuss the batch setting
\cite{burnaev2014efficiency,lei2013distribution,lei2014distribution},
the implicit goal in these is still to achieve a per-example error of at most $\alpha$.
This typically induces a large computational overhead in the inductive setting,
since every new point effectively requires fitting a new model.
The split-conformal method in \cite{lei2015conformal} fits a single model,
but at the cost of larger intervals \cite{chen2016trimmed}.
Our method, in contrast, focuses on an inductive batch setting
where the goal is to minimize the expected error,
and a predictor is trained only once on a fixed training set.

Several works consider a learning setting that is similar to ours,
most of which are motivated by specific applications.
In \cite{khosravi2011lower}, the authors design a multi-criteria loss and train a neural network
to explicitly predict the interval boundaries.
Several extensions of this approach have been suggested as well \cite{taormina2015ann,hosen2015improving}.
In contrast to our method, the above methods balance accuracy and interval size heuristically,
use a discontinuous (and non-convex) objective,
offer no statistical guarantees,
and are applied only to small datasets.



\section{Single-point vs. Batch} \label{sec:sp_vs_batch}
Before presenting our method, it will be useful to
formally discuss the subtle but crucial distinctions
between the single test-point and inductive batch settings.
Let $f : \X \rightarrow \R \times \R$ be an interval predictor.
In the single test-point setting
(as well as in the online learning setting used in conformal prediction \cite{shafer2008tutorial}),
the requirement is to guarantee that:
\begin{equation}
\probsub{\ysup{m+1} \in f(\xsup{m+1})}{D^{m+1}} \ge 1-\conf 
\label{eq:pi_erequirement_weak}
\end{equation}
where $D^{m+1}$ denotes the joint probability over both
$S$ and $(\xsup{m+1},\ysup{m+1})$. 
In this setting, and in order for Eq. \eqref{eq:pi_erequirement_weak} to hold,
the algorithm for constructing $f$
is typically given access to both $S$ \emph{and} $\xsup{m+1}$.
In contrast, in the batch setting, the algorithm is allowed access to $S$ alone.
To highlight this distinction we use the notations $f(\,\cdotp\,|\,S,\xsup{m+1})$ and $f(\,\cdotp\,|\,S)$ accordingly.

Methods which are designed to satisfy Eq. \eqref{eq:pi_erequirement_weak} 
can in principle be applied to the batch setting.
This, however, often comes at either a computational or a statistical cost.
One popular and straightforward approach 
is to compute a new predictor $f_x = f( \cdotp \,|\,S,x)$
for every new test point $x$.
For example, this idea lies at the core of conformal prediction,
and is used to provide regret-type guarantees for the online learning settings.
While feasible, such an approach carries a significant computational overhead.
To circumvent this, other approaches
partition $S$ into effective `train' and `test' subsets
\cite{burnaev2014efficiency,lei2015conformal,chen2016trimmed}.
This allows them to work with a single predictor,
but can result in the downgrading or loss of statistical guarantees.

One way to restore some guarantees in this setting
is to tighten the condition in Eq. \eqref{eq:pi_erequirement_weak},
and require that the probability of success conditionally holds for \emph{all}\footnote{
	The guarantees in \cite{lei2014distribution} are for \emph{almost} all $x \in \X$.}
examples \cite{lei2014distribution}, namely:
\begin{equation}
\probcondsub{y \in f(x\,|\,S)}{X=x}{D^{m},D_{Y|X}} \ge 1-\conf \quad\;
\forall x \in \X
\label{eq:pi_requirement}
\end{equation}
Here, the probability is over the joint distribution of sample sets $S$ and
conditional new labels $y$.
This means that for \emph{every} $x \in \X$,
the predicted intervals should contain the labels
with probability of at least $1-\conf$,
and is in effect what methods such as quantile regression \cite{koenker1978regression,koenker2005quantile}
aim for.

Due to its worst-case nature,
the performance of predictors constructed to satisfy Eq. \eqref{eq:pi_requirement}
is dominated by the difficult instances.
Here we argue that such worst-case requirements can be overly demanding.
In accordance, we propose an objective which modifies
Eq. \eqref{eq:pi_requirement} in two important ways.
First, we relax the worst-case requirement over $x \in \X$ to hold \emph{in expectation}
over the marginal $D_X$.
Second, in the analysis (\secref{sec:theory}), we partition $\conf$
into a confidence term $\delta$ over the training set
and an accuracy term $\epsilon$ over new examples.
Overall, our aim is to guarantee that,
with probability of at least $1-\delta$ over $S \sim D_{XY}^m$,
we have:
\begin{equation}
\probsub{y \in f(x\,|\,S)}{D_{XY}} \ge 1-\epsilon
\label{eq:batch_requierment}
\end{equation}
Intuitively, this means that when the training set is representative of the distribution,
we'd like $f$ to generalize well to new examples.
In the next section, we show how this can be achieved
by minimizing an appropriate loss function.


\section{Method} \label{sec:method}

Our approach considers the task of generating PIs from 
a model-free, discriminative learning perspective.
Given a training set
$S=\{ (\xsupi, \ysupi) \}_{i=1}^m$ sampled i.i.d. from $D_{XY}$,
our goal is to learn an interval predictor $f$
with a low expected loss:
\begin{equation}
\Loss{f} = \expect{D}{\loss{y,f(x)}}
\label{eq:expect_loss}
\end{equation}
The loss function $\loss{y,f(x)}$ should quantify how well the predicted interval $f(x)$
fits the label $y$.
Since a good interval is one which contains the true label,
a natural choice 
is the following \emph{interval 0/1-loss}:
\begin{equation}
\loss{y,\intrvl} = 
\loss{y,\l,\u} = 
\begin{cases*}
0 & if  $y \in \intrvl$  \\
1 & otherwise
\end{cases*}
\label{eq:01loss}
\end{equation}
Under this loss, Eq. \eqref{eq:expect_loss} can be rewritten as:
\begin{align}
\Loss{f} &= \E_{X} \expectcond{Y|X}{\1{y \in f(x)}}{X=x} \nonumber \\[1ex]
&= \E_{X} \Big[ {\probcondsub{y \not\in f(x)}{X=x}{Y|X}} \Big]  
\label{eq:expect_01loss}
\end{align}
Hence, a good mapping is one which produces good intervals
\emph{in expectation}.
Specifically, a mapping $f$ with $\Loss{f}=\alpha$ guarantees that predicted intervals
will contain the true labels with probability $1-\conf$, on average.

When comparing the above discriminative criterion
with the standard PI criterion in Eq. \eqref{eq:pi_requirement},
it becomes clear that they differ only in their requirement over $x$.
Specifically, while Eq. \eqref{eq:pi_requirement} requires
that labels be covered with probability $1-\conf$ for \emph{all} $x \in \X$,
Eq. \eqref{eq:expect_01loss} requires this only in expectation.
The discriminative objective in Eq. \eqref{eq:expect_01loss} can therefore be seen as 
a relaxation of the classic criterion,
one which allows for variability in the probability of covering the true label.
This gives our approach some flexibility, which as we show,
can boost predictive performance.

\subsection{Learning Objective} \label{subsec:01loss}
When considering how to train a predictor over $S$,
it might at first be tempting to solve the following 
empirical risk minimization (ERM) objective:
\begin{equation}
\min_{f \in \F} \frac{1}{m} \sum_i \loss{\ysupi,\sup{\lhat}{i},\sup{\uhat}{i}}
\label{eq:silly_erm}
\end{equation}
where $\intrvlhat = f(x)$ is the predicted interval,
and $\F$ is a class of interval predictors.
However, it quickly becomes apparent that 
without constraining the size of the predicted intervals, 
the loss can be arbitrarily low for \emph{any} input.
This is because, for any reasonable $\F$,
intervals can be made large enough to always contain the true labels.
This emphasizes a
fundamental tradeoff between error and interval size,
and motivates the addition of a global \emph{budget} constraint:
\begin{equation}
\begin{aligned}
& \min_{f \in \F} & & \frac{1}{m} \sum_i \loss{\ysupi,\lhati,\uhati} \\
& \text{s.t.}
& & \lhati \le \uhati \qquad \forall \, i=1,\ldots,m \\
& & & \frac{1}{m} \sum_i \uhati-\lhati \le \budget
\end{aligned}
\label{eq:01erm}
\end{equation}
Thus, for a given budget $\budget$,
our hypothesis class is effectively restricted to include
only predictors $f$ which generate intervals with an average size of at most $\budget$.
Of these, our objective is to choose one with minimal error.
Note that the $\lhat \le \uhat$ constraints
are always feasible when $\F$ includes functions with a bias term.

The idea of fixing a budget and minimizing error in some sense
reverses their roles when compared to that of classic PI methods.
In these, 
the first step is to predetermine and fix the 
amount of tolerated error $\conf$.
Then, for a new point $x$,
the interval is set to the tightest
lower and upper bounds which contain $1-\conf$ of the 
(estimated) conditional density of $Y$ given $X$.\footnote{
    Typically under additional constraints, such as symmetric or one-sided error tails.}
Compared to this approach,
the roles of the objective and constraints in Eq. \eqref{eq:01erm} are reversed.

Nonetheless, our method can also be applied to the fixed-error setting in two ways.
First, for a given $\conf$, it is possible to solve 
a ``reversed'' variant of Eq. \eqref{eq:01erm}:
\begin{equation}
\begin{aligned}
& \min_{f \in \F} & &  \sum_i \uhati-\lhati \\
& \text{s.t.}
& & \lhati \le \uhati \qquad \forall \, i=1,\ldots,m \\
& & & \frac{1}{m} \sum_i \loss{\ysupi,\sup{\lhat}{i},\sup{\uhat}{i}} \le \conf
\end{aligned}
\label{eq:01erm_reverse}
\end{equation}
When $\loss{\cdotp}$ is convex (such as in the surrogate we consider in next section),
Eqs. \eqref{eq:01erm} and  \eqref{eq:01erm_reverse} are in fact dual,
in the sense that
for every error $\conf$ there exists a budget $\budget$
such that the solutions of Eq. \eqref{eq:01erm} and Eq. \eqref{eq:01erm_reverse} coincide.
A second option is to perform a line search over $\budget$ using Eq. \eqref{eq:01erm},
and estimate $\conf(\budget)$ over a held-out validation set.
Since the error is monotone in the budget, this can be done efficiently.
Note that a budget-error curve can be constructed by 
solving Eq. \eqref{eq:01erm} for a range of budgets (as in Fig.~\ref{fig:plots}(B)).

\subsection{Convex Surrogate} \label{subsec:surrogate_loss}
Since the labels and predictions in our setting are real,
the problem we consider may appear to be one of regression.
Nonetheless, it is in essence a classification problem,
and the combinatorial nature of the 0/1 loss in Eq. \eqref{eq:01loss}
makes it NP-hard to solve (see \secref{subsec:nphard}).
Due to this, we turn to learning with tractable surrogates.
In what follows we suggest a convex surrogate 
to the 0/1 loss, and discuss its properties.


Our surrogate loss is inspired by the \emph{$\varepsilon$-insensitive loss}
used in Support Vector Regression Machines \cite{vapnik2013nature,smola2004tutorial}:
\begin{align}
\lossprxy{y,\yhat ; \varepsilon} &= 
\max \{ 0, |y-\yhat|-\varepsilon \} \nonumber \\
&= \begin{cases*}
0 & if  $|y-\yhat| \le \varepsilon$  \\
|y-\yhat| & otherwise
\end{cases*}
\label{eq:eps_insensitive}
\end{align}
Here, a point-prediction $\yhat \in \R$ incurs no penalty if it is within
distance $\varepsilon$ of the true label $y$;
otherwise, the penalty is linear.
For the special case of $\varepsilon=0$, the $\varepsilon$-insensitive
reduces to the standard mean absolute error loss.
For $\varepsilon=1$, the loss $\lprxy$ can be thought of
as a symmetrized variant of the popular \emph{hinge loss} used in SVMs,
in which both under- and over-estimates of $y$ are penalized. 
This motivates the idea that,
just as the hinge loss serves as a proxy to 
the binary 0/1-loss,
the $\varepsilon$-insensitive loss can be used as a surrogate
to the interval 0/1-loss. 

Recall that our learning goal is to produce a low-error
interval predictor under a budget constraint
on the average interval size (Eq. \eqref{eq:01erm}).
For a given budget $\budget$,
one way to achieve this is to set $\varepsilon=\budget$
and learn a point-predictor $\fprxy : \X \rightarrow \R$ with $\lprxy$.
Then, predicted intervals are constructed via
$\lhat= \fprxy(x) - \Delta/2$ and
$\uhat = \fprxy(x) + \Delta/2$,
for $\Delta=\varepsilon$.
Note that for every $x$, the interval size is exactly $\Delta = \uhat-\lhat = B$.
Under this approach,
the training loss and test-time predictions are calibrated: 
a predicted interval is penalized only if it does not contain the true label.

Clearly, fixing all interval sizes $\Delta$ to $\budget$
is sufficient for satisfying the budget constraint.
This, however, is not necessary, as the constraint requires only that
the intervals be of size $\budget$ \emph{on average}.
Hence, interval sizes can vary,
and solutions with varying interval size can in principle 
give lower error.

The idea of varying interval sizes lies at the base of our approach.
To allow for variation, we \emph{parametrize} the interval size $\Delta$,
and jointly learn a point predictor $\yhat=f(x;w)$
and an interval-size predictor $\Deltahat=g(x;v)$. 
For learning, we use a parametrized
$\varepsilon$-insensitive loss, with a per-example insensitivity scale
$\varepsilon(x) = g(x;v)$.
The learning objective can be written as follows:
\begin{equation}
\begin{aligned}
& \min_{w,v} & & \frac{1}{m} \sum_i
\lossprxy{\ysupi,\sup{\yhat}{i} \,;\, \sup{\Deltahat}{i}/2}
\\
& \text{s.t.}
& & \sup{\Deltahat}{i} \ge 0 \qquad \forall i=1,\ldots,m \\
& & & \frac{1}{m} \sum_i \sup{\Deltahat}{i} \le \budget
\end{aligned}
\label{eq:proxy_loss}
\end{equation}
where in practice $f$ and $g$ are additionally regularized.
For a liner parameterization, namely $f(x;w)=\inner{w,x}+b$ and $g(x;v)=\inner{v,x}+a$
for $w,v \in \R^d$ and $b,a \in \R$,
the loss and constraints in Eq. \eqref{eq:proxy_loss} become convex in both $w$ and $v$.
Finding the global optimum can  be done efficiently using standard convex solvers.

Although $\yhat$ lies at the center of $\Deltahat$,
the quantiles which correspond to the interval endpoints $\lhat,\uhat$
are not necessarily symmetric, and can vary across examples.
This is in contrast to most methods which use
fixed symmetric-tail quantiles.
A possible alternative is to directly parametrize
the interval boundaries via $\lhat=f_\l(x;w_\l)$ and $\uhat=f_\u(x;w_u)$.
In the supplementary material we show that for linear predictors
both parameterizations are equivalent (up to regularization).

The objective in Eq. \eqref{eq:proxy_loss} allows for variability
in the size of predicted intervals.
As in other approaches, this can be used to account for
conditional heteroskedasticity.
But, more importantly, a differential allocation of the budget allows for
variability in the \emph{conditional probability of errors}, namely:
\begin{equation}
\P_{Y|X}[ y \not\in [\yhat-\Deltahat/2,\yhat+\Deltahat/2] \,\,|\,\, X=x]
\end{equation}
Intuitively, this allows our method to boost the overall accuracy
by ``sacrificing'' the accuracy of some points (by reducing the size of their interval)
in favor of others (by allocating them more of the budget).

Without restricting the form of $\Deltahat$,
solving Eq. \eqref{eq:proxy_loss} would most likely result
in overfitting.
In this scenario,
intervals would either be allocated just enough
budget to cover the labels exactly, or allocated no budget at all.
This outcome is clearly undesired as it can result in unstable predictions.
It hence remains to show that learning parametrized interval predictors
can lead to good generalization.
In the next section we analyze the sample complexity
of learning in our setting,
and show that the VC dimension of a class of interval predictors
is linear in the VC dimension of a corresponding class of threshold functions.


\section{Theoretical Analysis} \label{sec:theory}
In this section we consider three aspects of our approach:
the asymptotic properties of the interval constraints in Eq. \eqref{eq:01erm},
the sample complexity of learning with the 0/1-interval loss Eq. \eqref{eq:01loss},
and the computational hardness of ERM (Eq. \eqref{eq:01erm}).

\paragraph{Constraints:}
Eq. \eqref{eq:01erm} requires that 
the average interval size does not exceed the budget $\budget$,
and that each predicted interval is consistent (that is, $\lhati \le \uhati$).
For the budget constraint,
let $D$ be a random variable specifying the size of predicted intervals,\footnote{
The source of randomness for $D$ are the samples $(\xsupi, \ysupi)$ in $S$,
which determine both
the optimal $f$ and the prediction outcomes $f(\xsupi)=\intrvltmplt{\lhati}{\uhati}$.}
with instances denoted by $d$. 
The distance between the average and expected interval sizes
$|\frac{1}{m}\sum_{i=1}^m d_i - \expect{}{D}|$
can therefore be bounded using standard concentration bounds.
For consistency, note that if functions in $\F$ include a bias term,
then the constraint can always be satisfied.
This means that given an empirically consistent predictor $f$,
its expected consistency can be estimated using 
sample complexity bounds for the realizable case.
In practice, inconsistent predicted intervals can
simply be replaced by point predictions.
Alternatively, both constrains can be replaced
by a single convex constraint of the form
$\frac{1}{m} \sum_i \max\{0,\uhati-\lhati\} \le B$.

\paragraph{Sample complexity:}
As noted in \secref{sec:intro},
classic PI methods consider a single confidence term $\alpha$
defined the joint distribution of training set \emph{and} a new example.
For our analysis, we separately consider the training set and new examples.
We use PAC theory to determine, for every $\epsilon,\delta \in [0,1]$,
the minimal number of examples $m$
that guarantee an expected error of at most $\epsilon$ (over the test distribution)
with probability of at least $1-\delta$ (over the train set).
To this end, in the following \secref{subsec:vcdim}
we analyze the VC dimension of learning a 
class of interval predictors $\F$. 

\paragraph{Computational hardness:}
Many of the discontinuous losses in machine learning
lead to combinatorial optimization problems which are hard to optimize.
In \secref{subsec:nphard} we prove that this is also the case for
the interval 0/1 loss in Eq. \eqref{eq:01loss}.
To this end, we show a reduction from the NP-hard problem MAX FS,
which motivates the use of the convex proxy loss
in Eq. \eqref{eq:proxy_loss}.

\subsection{VC Dimension} \label{subsec:vcdim}
In this section we analyze the VC dimension
of batch interval prediction.
The difficulty in analysis lies in the fact that
the function classes we consider are not binary, as
both labels and predictions in our setting are real.
Albeit the regressive nature of the learning setup,
the loss is binary,
and as a result, the complexity of learning is combinatorial.
To overcome this difficulty,
the main modeling point here is that we consider \emph{both $x$ and $y$}
(and not just $x$) as the input to an hypothesis.
This allows for expressing the VC dimension of a class of interval predictors
in terms of the VC dimension of some base class of threshold functions.

Let $\Base =\{ \base:\X \rightarrow \R \}$
be a \emph{base} class of real functions (e.g., linear functions).
Denote by $\F(\Base)$ the class of interval predictors
whose interval boundaries are set by functions in $\Base$,
namely function of the form
$f_{\base_\l,\base_\u}(x) = \intrvltmplt{\base_\l(x)}{\base_\u(x)}$,
where $\base_\l,\base_\u \in \Base$.
Our goal here is to bound the VC dimension of $\F(\Base)$.
We show that this can be done by first considering 
the VC dimension of threshold functions over the base class $\Base$.

In binary classification,
a conventional way of using real functions $\base \in \Base$
for classification is to consider
the corresponding classes of threshold functions:
\begin{align*}
\Thresh_{\le}(\Base) &= \{ \1{\base(x) \le \theta} \,:\, \base \in \Base\} \\
\Thresh_{\ge}(\Base) &= \{ \1{\base(x) \ge \theta} \,:\, \base \in \Base\} 
\end{align*}
For example, when $\Base$ is the set of linear functions,
both $\Thresh_\le(\Base)$ and $\Thresh_\ge(\Base)$ are equivalent to
the standard class of halfspace classifiers, 
whose VC dimension is $d+1$.

The VC dimension of a class $\Thresh$ of binary classifiers
(such as $\Thresh_\le(\Base)$ and $\Thresh_\ge(\Base)$)
regards the number of possible label assignments
of functions $\thresh \in \Thresh$ to a set of $m$ points $\{\xsupi\}_{i=1}^m$.
The interval class $\F(\Base)$, however, is not binary, as functions
$f\in\F(\Base)$ map examples to real intervals.
Nonetheless, the true object of interest in terms of sample complexity
is rather the number of assignments of the \emph{loss function} over
a given class.
This means that, for a given sample set $S=\{(\xsupi,\ysupi)\}_{i=1}^m$ of size $m$,
we will analyze the number of possible assignments
to the interval 0/1 loss in Eq. \eqref{eq:01loss}
when learning with $\F(\Base)$.
Formally, with a slight abuse of notation,
we analyze the VC dimension of the binary class:
\begin{equation}
\L(\F) = \{ \loss{y,f(x)} \,:\, f \in \F \}
\label{eq:vc_loss}
\end{equation}

Note that analyzing the complexity of a class under
a loss function is in fact the true (though sometime implicit) goal
in standard binary classification as well.
We prove the following in the supplementary material:
\begin{lemma}
	The VC dimension of $\Thresh$ equals
	the VC dimension of $\L_{0/1}(\Thresh)$.
\end{lemma}
where $\L_{0/1}$ is appropriately defined over the 
standard binary 0/1 loss $\losssub{0/1}{y,\yhat}=\1{y\neq\yhat}$.

The main difference in the analysis of these classes
lies in the input and output of the functions they include.
For binary classifiers, the inputs are examples $x$,
and the outputs (which we'd like to shatter) are labels $y$.
In contrast,
inputs to functions in $\L(\F)$ and $\L_{0/1}(\Thresh)$ are \emph{pairs} $(x,y)$
of an example and a label,
while outputs are the binary evaluations of the loss function.

The next theorem shows that the VC dimension of $\F(B)$
under $\Loss{\cdotp}$
can be linearly bounded by the
VC dimension of a corresponding binary threshold class.
For simplicity, we focus on the case where
$\Thresh(\Base)=\Thresh_\le(\Base)=\Thresh_\ge(\Base)$,
which holds under mild assumptions.
The general case is straightforward.
\begin{theorem}
Let  $\Base =\{ \base:\X \rightarrow \R \}$ be a base class.
If the VC dimension of the threshold class $\Thresh(\Base)$ is $k$,
then the VC dimension of the interval class $\F(\Base)$
over the interval 0/1-loss in Eq. \eqref{eq:01loss} is at most $10k$.
\end{theorem}

\begin{proof}
For some $f_{\base_\l,\base_\u} \in \F(\Base)$, we have:
\begin{align}
\loss{y,f_{\base_\l,\base_\u}(x)} & = 
\begin{cases*}
0 & if  $\base_\l(x) \le y \wedge \base_\u(x) \ge y$ \\
1 & otherwise
\end{cases*}  \nonumber \\
& = 1 - \1{\base_\l(x) \le y} \cdotp  \1{\base_\u(x) \ge y} 
\label{eq:complement}
\end{align}
%
Consider a set $S=\{(\xsupi,\ysupi)\}_{i=1}^m$ of size $m$
that $\L(\F)$ shatters.
On the one hand, the $m$ pairs have $2^m$
assignments under $\L(\F)$.
On the other hand, due to Eq. \eqref{eq:complement},
the number of assignments is at most
that of $\Thresh_\le(\Base)$ times that of $\Thresh_\ge(\Base)$,
namely:
\begin{equation*}
\Pi_m \left( \L(\F(\Base)) \right) = 2^m \le
\Pi_m \left( \Thresh_\le(\Base) \right) \cdotp \Pi_m \left( \Thresh_\ge(\Base) \right)
\end{equation*}
where $\Pi_m(\cdotp)$ denotes the maximal number of assignments
of an hypothesis class over $m$ points.
Using the Sauer-Shelah Lemma and our assumptions
on the VC dimension of $\Thresh(\Base)$,
for $\vcdim \le m/3$ we have:
\begin{equation}
2^m \le
\left( \sum_{i=0}^\vcdim \binom{m}{i} \right)^2
\le \left( \frac{em}{\vcdim} \right)^{2\vcdim}
\label{eq:sauer}
\end{equation}
Solving for the maximal $m$ satisfying Eq. \eqref{eq:sauer}
gives $m \le 10\vcdim$.
Hence, $\text{VC}(\F(\Base))$ under $\L$
is at most $10\vcdim$.

\end{proof}

\begin{corollary}
Let $\Thresh(\Base)$ be a class of binary classifiers over the base class $\Base$
with VC-dimension $k$, and let $\F(\Base)$ be the corresponding interval class.
Then, for every $\epsilon, \delta \in [0,1]$,
if $\Thresh(\Base)$ is $(\epsilon,\delta)$-PAC-learnable over the \emph{binary} 0/1-loss
with $m(k)$ samples,
then F(B) is $(\epsilon,\delta)$-PAC-learnable over the \emph{interval} 0/1-loss
with $m(10k)$ samples.
Hence, for any $m' \ge m(10k)$:
\begin{equation}
\probsub{
\probsub{y \notin f(x\,|\,S)}{D_{XY}} \le \epsilon
}{S \sim D_{XY}^{m'}} \ge 1-\delta
\label{eq:corollary}
\end{equation}
\end{corollary}

\subsection{NP-hardness} \label{subsec:nphard}
In this section we establish the computational hardness
of minimizing the empirical risk over the interval 0/1-loss
as in Eq. \eqref{eq:01erm}.
We focus on the linear case where $x \in \R^d$
and the interval class is:
\begin{align*}
\F_{\text{lin}} &= \left\{ f(x) = \intrvltmplt{\base_\l(x)}{\base_\u(x)} \,:\, \base_\l,\base_\u \in \Base_{\text{lin}} \right\}, \\
\Base_{\text{lin}} &= \left\{ \base(x) = \inner{w,x}+c \,:\, w \in \R^d,\, c \in \R \right\} 
\end{align*}

\begin{theorem}
Solving Eq. \eqref{eq:01erm} over $\F_{\textup{lin}}$ is NP-hard.
\end{theorem}

%

\begin{proof}
We show a reduction 
from the NP-hard problem \maxfs\ \cite{sankaran1993note}.
In this problem,
given a (not necessarily feasible) linear system $Az=d$
of $M$ equations over $N$ variables,
the goal is to find the maximal feasible subsystem.\footnote{
The popular version of \maxfs\ considers
the system $Az \le d$, but is also NP-hard for the relations $\{<,=,\neq\}$
\cite{amaldi1995complexity}.}
Given an instance of \maxfs, namely $A \in \R^{M \times N}$
and $d \in \R^M$, we will construct a training set $S=\{(\xsupi,\ysupi)\}_{i=1}^m$
and budget $\budget$ 
with corresponding optimal solutions.
W.l.o.g. we will assume that rows in $A$ are distinct and that $M\ge2$.

Our construction is as follows.
Let $m=2M+1$. 
For every $i=1,\dots,M$, set $\xsupi=(A_{i1},\dots,A_{iN})$ and $\ysupi=d_i$,
for $i=M+1,\dots,m$ set $\xsupi=(0,\dots,0)$ and $\ysupi=0$,
and set $\budget=0$.
Let $f^*(x)=\intrvltmplt{\base_\l(x)}{\base_\u(x)}$ be a minimizer of Eq. \eqref{eq:01erm},
with $\base_\l(x)=\inner{w_\l,x}+c_\l$ and $\base_\u(x)=\inner{b_\u,x}+c_\u$.
First, note that $f^*$ must have $c_\l=c_\u=0$.
This is because a solution with no bias terms is correct
on at least $M+1>m/2$ examples, while any other solution can be correct
on at most $M \le m/2$ examples.
Second, in order to satisfy the budget and interval constraints,
it must also hold that $w_\l=w_\u=w$.
This means that for $i=1,\dots,M$, we have $f^*(\xsupi)=\ysupi$
(and therefore $\loss{\ysupi,f^*(\xsupi)}=0$)
iff the $i^{\text{th}}$ equation in $Az=d$
is satisfied by $z=w$, concluding our proof.
\end{proof}


\section{Experiments} \label{sec:experiments}

\begin{figure*}[!t]
	\begin{center}
		\includegraphics[width=\textwidth]{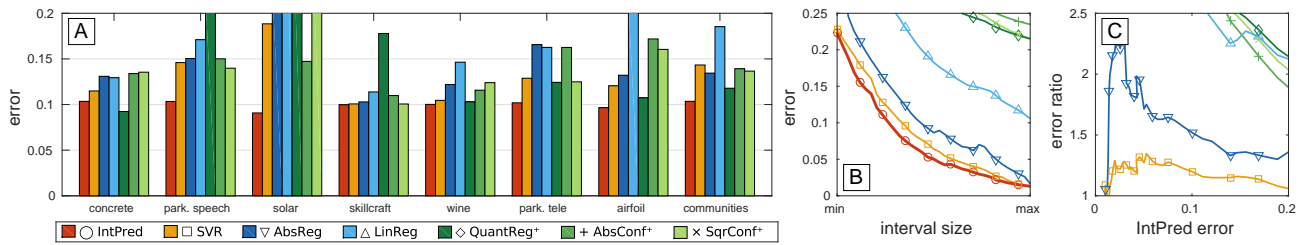}
		\caption{
(A) Accuracy across datasets,
for budgets $\budget$ under which the test error of 
\ourmthd\ is roughly 0.1.
(B)~Test error per normalized interval size, averaged over datasets.
(C) Ratio of baseline test errors vs. \ourmthd.
}
		\label{fig:plots}
	\end{center}
\end{figure*}

In this section we evaluate the performance of our method
on a collection of benchmark datasets.
We compare our method to several baselines on data from the UCI repository \cite{lichman2013uci}.
We used all dense-featured regression datasets with between 1,000 and 10,000 samples,
giving a total of 8 datasets.
All of our experiments use a 80:20 train-test split.
For each considered setting
we set all meta-parameters via 5-fold cross-validation on the training set over a tight grid.
We report average errors (1-accuracy) over 100 random data splits.
For datasets with less than 25 features, we augment each example with
pairwise terms.

At its core,
our method offers a general means for exploiting the variability in interval size
for reducing error.
While applicable to virtually any base class of real functions,
our evaluation focuses on linear functions.
This is because linear functions can be plugged into almost any method,
provide computational guarantees for our method as well as the other baselines,
and in many cases work well in practice.
By using a linear base class in all methods,
we allow for a fair comparison, where results express the statistical power
(rather than computational traits) of the different methods.

In accordance with the above,
we compare our method (\ourmthd) 
to several baselines,
some of which are designed for the fixed-error setting
(and take as input a significance lever $\conf$),
and some, likes ours, which are designed for the fixed-budget setting
(and take a budget $\budget$ as input).
Our fixed-error baselines include quantile regression (\method{QuantReg})
with symmetric tails, and batch-mode conformal prediction
with the absolute loss (\method{AbsConf}) and squared loss (\method{SqrConf}).
We use the efficient split-conformal inference method suggested in \cite{lei2015conformal,lei2017distribution}, since other conformal methods
are computationally infeasible for the datasets we consider.
Our fixed-budget baselines include
regressing with the absolute loss (\method{AbsReg}) or squared loss (\method{LinReg})
augmented with a fixed symmetric interval of size $\budget$,
as well as the fixed-budget method (\method{SVR})
discussed in \secref{subsec:surrogate_loss},
which minimizes the $\varepsilon$-insensitive loss with $\varepsilon=\budget$.
We use $L_2$ regularization in methods where this is possible.

Our main evaluation criterion is the test error,
namely the probability that
a predicted interval does not contain the true test label.
Since errors can be reduced simply by enlarging the intervals,
for a fair comparison we evaluated all methods over a fixed
set of average interval sizes $\{\budget_i\}$,
and compared for each $\budget_i$ separately.
This is easily achieved for the fixed-budget methods
(since the budget is part of the input),
but not necessarily so for the fixed-error methods,
which do not offer a direct way to control interval sizes.
To include these in the comparison, we first evaluate them
over a tight set of confidence levels $\alpha_j \in [0,1]$.
Then, we interpolate errors from the interval size outputs,
and compute approximate errors for the budgets $\budget_i$.
For interpolation, we used $3^{\text{rd}}$-order
monotonically-increasing concave splines with 5 knots over 30 points.
This gave average $R^2$ values of
0.996 for \method{QuantReg}, 0.855 for \method{AbsConf}, and 0.901 for \method{SqrConf}.

Our results are presented in Fig. \ref{fig:plots}.
In the spirit of classic PI methods, we focus on the high-confidence regime
with errors in the range $[0,0.2]$.\footnote{
Both approaches present some difficulty when targeting very low test errors.
For the fixed-budget methods,
a higher budget leads to lower train errors.
Hence, budgets which give a train error of 0 may be prone to overfitting.
For most fixed-error methods,
as there are typically no finite-sample guarantees on test error,
even setting $\conf=1$ does not guarantee an arbitrarily low test error.
}
Fig. \ref{fig:plots} (A) presents an evaluation across all datasets.
For each dataset, we show results for the $\budget_i$ under which the test accuracy
of our method was roughly 0.1.
As can be seen, \ourmthd\ outperforms all other baselines
for all but one dataset (where it ranks a close second).
Across all evaluations,
\ourmthd\ is statistically significantly better than the other methods
(Friedman test \cite{demvsar2006}, $P<10^{-10}$).
While other methods perform well over some datasets,
their performance is in general inconsistent.
This is demonstrated in Fig. \ref{fig:plots} (B)
which displays errors averaged over all datasets for 
a range of (normalized) budgets.\footnote{Since the datasets vary considerably
	in the scale and variance of label values,
	we present results for normalized budgets,
	where a normalized $\budget=1$ corresponds (approximately) to an error of 0.1.}
As the plot shows, the volatility of some baselines
(especially of the fixed-error methods)
results in high average errors.

Recall that, in effect, \method{SVR} optimizes over the same interval loss as \ourmthd,
but is constrained to fixed-size interval predictions.
Hence, comparing \ourmthd\ to \method{SVR} quantifies
the added value (in terms of accuracy) of allowing for
variability in interval size.
Fig.~\ref{fig:plots}(C)  presents a direct comparison
between the average error of \ourmthd\ and that of other methods,
where each point on the plot corresponds to a different budget $\budget_i$.
The relative performance of \method{SVR} 
demonstrates that interval-size flexibility reduces the error by roughly 20\%
across most of the focal error region.

\section{Discussion} \label{sec:discussion}
This paper presents a method for constructing
prediction intervals in an inductive batch setting.
Our approach views PI estimation as a discriminative learning task,
where the goal is to minimize the probability of error under a budget constraint.
The algorithmic, computational, and statistical results
were made possible by modifying the classic PI objective:
focusing on expected errors, allowing for error variability,
and separating accuracy from confidence.


Our experimental evaluation empirically demonstrates
the potential of allowing for variability in accuracy across examples.
As in many discriminative methods,
this boost in accuracy comes at the price of explainability,
as our method does not offer any guarantees on
the confidence associated with individual examples.
For example, a predicted interval can be small either because
the confidence is very high (and there is no reason to waist budget),
or because it is very low (and allocating budget is justified in terms of the loss).
One interesting direction worth exploring is that of augmenting
each prediction with a confidence estimate,
namely the probability that the true label is within the interval.
We leave this for future research.


%


\paragraph{Acknowledgments}
Supported in part by a grant from the Israel Science Foundation, a
grant from the United States-Israel Binational Science Foundation
(BSF), and the Israeli Centers of Research Excellence (I-CORE)
program (Center No. 4/11).
NR would like to thank Alon Gonen for his moral support, good advice, and technical contribution.

\bibliographystyle{acm}
\bibliography{confintrvls}


\newpage
\part*{Supplamentary Material}

\section{Equivalence of Parameterization}
There are two natural parameterizations of intervals.
The first is to model the interval boundaries via:
\[
\lhat=f_\l(x;w_\l), \qquad
\uhat=f_\u(x;w_\u)
\]
The second is to model the interval center and size:
\[
\yhat=f(x;w), \qquad
\Deltahat=g(x;v)
\]

Here we show that for linear predictors, both parameterizations are equivalent.
Specifically, we show that for every $(w_\l,w_\u)$
there exist some $(\tilde{w},\tilde{v})$ whose predictions coincide on all $x$,
and vice versa.

Let $w_\l,w_\u \in \R^d$, and consider some $x \in \X$.
Note that:
\begin{align*}
\yhat &= \frac{1}{2}(w_\l^\top x+w_\u^\top x) =
\big( \frac{1}{2}(w_\l+w_\u) \big)^\top x =
\tilde{w}^\top x \\
\Deltahat &= w_\u^\top x - w_\l^\top x =
(w_\u-w_\l)^\top x =
\tilde{v}^\top x
\end{align*}

Similarly, for $w,v \in \R^d$, we have:
\begin{align*}
\lhat &= w^\top x - \frac{1}{2} v^\top x = 
(w-\frac{1}{2}v)^\top x =
\tilde{w_\l}^\top x \\
\uhat &= w^\top x + \frac{1}{2} v^\top x = 
(w+\frac{1}{2}v)^\top x =
\tilde{w_\u}^\top x
\end{align*}

We note that the only practical difference here would be
when each component is regularized separately (e.g., with a different
regularization constant).

\section{VC of Functions and Losses}
Recall that for a loss function $L$ and a function class $\F$,
we define their composition as:
\begin{equation}
\L(\F) = \{ \loss{y,f(x)} \,:\, f \in \F \}
\label{eq:vc_loss}
\end{equation}
we denote the by $\L_{0/1}(\Thresh)$
the composition of the binary 0/1-loss function $\losssub{0/1}{\cdotp}$
with a binary function class $\Thresh$,
where:
\begin{equation}
\losssub{0/1}{y,\yhat}=\1{y\neq\yhat}
\label{eq:01loss}
\end{equation}

Here we prove the following claim:
\begin{lemma}
The VC dimension of $\Thresh$ equals
the VC dimension of $\L_{0/1}(\Thresh)$.
\end{lemma}

\begin{proof}
The important observation here is that while both $\Thresh$ and $\L_{0/1}(\Thresh)$
include functions with binary outputs, the functions differ in their domain.
Specifically, functions in $\Thresh$ map items $x$ to binary outputs $y \in \{0,1\}$,
while functions in $\L_{0/1}(\Thresh)$ take as input pairs $(x,y)$ with $y \in \{0,1\}$
and, via some $\thresh \in \Thresh$, output the loss value $z \in \{0,1\}$.
	
Assume the VC dimension of $\Thresh$ is $m$,
then there exist some $x_1,\dots,x_m$ which shatter $\Thresh$.
This means that for every $y_1,\dots,y_m$, there exists some $h_y \in \Thresh$
for which $h_y(x_i)=y_i$ for every $i$.
Consider the set of pairs $(x_1,0),\dots,(x_m,0)$.
For any $z_1,\dots,z_m$, we have some $h_z \in \Thresh$ for which $h_z(x_i)=z_i$
for every $i$.
This means that any $z_i=0$ gives $h_z(x_i)=0$, and hence:
\[
\losssub{0/1}{0,h_z(x_i)} = 0 = z_i
\]
Similarly, for any $z_i=1$ we have $h_z(x_i)=1$ and:
\[
\losssub{0/1}{0,h_z(x_i)} = 1 = z_i
\]

Now, assume the VC dimension of $\L_{0/1}(\Thresh)$ is $m$.
Then there exist some $(x_1,y_1),\dots,(x_m,y_m)$ which shatter $\L_{0/1}(\Thresh)$.
This means that for every $z_1,\dots,z_m$, there exists some $h_z \in \Thresh$
such that $\losssub{0/1}{y_i,h_z(x_i)} = z_i$.
Consider the set $x_1,\dots,x_m$.
For any $y_1,\dots,y_m$, set $z_i=0$ for all $i$.
Hence, the corresponding $h_z$ is such that $\losssub{0/1}{y_i,h_z(x_i)} = 0$ for all $i$, 
which means that $h(x_i)=y_i$ as needed.
\end{proof}


\end{document}